\documentclass{article}

\usepackage[preprint]{neurips_2026}

\setcitestyle{round}

\usepackage[utf8]{inputenc} 
\usepackage[T1]{fontenc}    
\usepackage{url}            
\usepackage{booktabs}       
\usepackage{amsfonts}       
\usepackage{nicefrac}       
\usepackage{microtype}      
\usepackage{xcolor}         

\usepackage{amsmath,amssymb,bbm}
\usepackage{mathrsfs}
\usepackage{mathtools}
\usepackage{enumitem}
\usepackage{wrapfig}
\usepackage{mdframed}

\usepackage{pgfplots}
\pgfplotsset{compat=1.18}

\usepackage{xcolor}
\definecolor{mydarkblue}{rgb}{0,0.08,0.45}

\usepackage[
  colorlinks=true,
  linkcolor=mydarkblue,
  citecolor=mydarkblue,
  urlcolor=mydarkblue,
  filecolor=mydarkblue
]{hyperref}

\newcommand{\blackref}[1]{\hyperref[#1]{\textcolor{black}{\ref*{#1}}}}

\newcommand{\mc}[1]{\mathcal{#1}}

\newcommand{\mb}[1]{\mathbb{#1}}

\newcommand{\tr}{\operatorname{tr}}

\DeclareMathOperator*{\argmin}{argmin}

\newcommand{\BlackBox}{\ensuremath{\square}}

\newenvironment{proof}[1][]
{
\par\noindent
\textbf{Proof
\if\relax\detokenize{#1}\relax
\else\ #1
\fi. }\ignorespaces
}
{
\hfill\BlackBox\par\vspace{2mm}
}

\newtheorem{theorem}{Theorem}[section]
\newtheorem{lemma}[theorem]{Lemma}
\newtheorem{proposition}[theorem]{Proposition}
\newtheorem{remark}[theorem]{Remark}
\newtheorem{corollary}[theorem]{Corollary}
\newtheorem{definition}[theorem]{Definition}

\title{Symmetry Guarantees Statistic Recovery \\  in Variational Inference}

\author{
  Daniel Marks \\
  Department of Computer Science\\
  University of Oxford\\
  \texttt{daniel.marks@cs.ox.ac.uk} \\
  \And
  Dario Paccagnan  \\
  Department of Computing \\
  Imperial College London \\
  \texttt{d.paccagnan@imperial.ac.uk} \\
  \AND
  Mark van der Wilk
  \\
  Department of Computer Science \\
  University of Oxford \\
  \texttt{mark.vdwilk@cs.ox.ac.uk} 
}

\begin{document}

\maketitle

\begin{abstract}
    Variational inference (VI) is a central tool in modern machine learning, used to approximate an intractable target density by optimising over a tractable family of distributions. As the variational family cannot typically represent the target exactly, guarantees on the quality of the resulting approximation are crucial for understanding which of its properties VI can faithfully capture. Recent work has identified instances in which symmetries of the target and the variational family enable the recovery of certain statistics, even under model misspecification. However, these guarantees are inherently problem-specific and offer little insight into the fundamental mechanism by which symmetry forces statistic recovery. In this paper, we overcome this limitation by developing a general theory of symmetry-induced statistic recovery in variational inference. First, we characterise when variational minimisers inherit the symmetries of the target and establish conditions under which these pin down identifiable statistics. Second, we unify existing results by showing that previously known statistic recovery guarantees in location--scale families arise as special cases of our theory. Third, we apply our framework to distributions on the sphere to obtain novel guarantees for directional statistics in von Mises--Fisher families. Together, these results provide a modular blueprint for deriving new recovery guarantees for VI in a broad range of symmetry settings.
\end{abstract}

\section{Introduction}
One of the central problems of modern statistics and machine learning is the accurate estimation of intractable probability distributions. It most prominently arises in the context of Bayesian statistics, where the intractability of the marginal likelihood renders exact computation of the posterior distribution infeasible in all but the simplest statistical models. This has led to the development of a range of approximate inference methods to approximate or sample from such intractable distributions.

Variational inference (VI) has emerged as a popular approximate inference technique \citep{jordan_introduction_1999, wainwright_graphical_2008, blei_variational_2017}, balancing the scalability of simpler methods such as Laplace approximations \citep{mackay_bayesian_1992} with the fidelity of more computationally intensive approaches such as MCMC \citep{robert_monte_2004}. At its core, VI recasts approximate inference as an optimisation problem: given an intractable target distribution, it seeks the best approximation to it within a tractable variational family by minimising a measure of dissimilarity between the two. The most commonly employed objective is the reverse Kullback--Leibler (KL) divergence, for which a readily computable optimisation surrogate is provided by the evidence lower bound. However, the principles of VI are agnostic to the particular choice of divergence, and alternative objectives have been successfully explored, including the forward KL divergence \citep{naesseth_markovian_2020}, \(\alpha\)-divergences \citep{hernandez-lobato_black-box_2016, dieng_variational_2017, daudel_alpha-divergence_2023}, \(\alpha\)--\(\beta\) divergences \citep{regli_alpha-beta_2018}, and R\'enyi divergences \citep{li_renyi_2016}.

A crucial drawback of variational methods lies in the fact that the variational family may not contain the target distribution, in which case the resulting approximation can be arbitrarily bad. This occurs, for example, when the variational family cannot match the target’s number of modes, or is restricted to factorised distributions, as in mean-field variational inference \citep{peterson_mean_1987, hinton_keeping_1993}, despite the target exhibiting strong dependencies. As a result, guarantees on the quality of VI's estimates are crucial for assessing the reliability of the resulting approximations and understanding when and why they may be relied upon in practice.

In this work, we derive formal guarantees for the statistic recovery of variational approximations, even under severe model misspecification, in the presence of symmetries of the target distribution. Our goal is to develop a general framework for symmetry propagation in variational inference, together with a precise correspondence between symmetries and recoverable statistics, that allows us to answer the following open problem first articulated by \citet{margossian_variational_2025}:

\begin{center}
\emph{For a given symmetry of the target distribution,\\
which statistics does variational inference recover exactly?}
\end{center}

\paragraph{Main contributions}

\begin{enumerate}
    \item We prove a general statistic recovery theorem showing that if the target is invariant under a group of symmetries that the variational family respects, a unique variational minimiser inherits this invariance and thus recovers all symmetry-determined statistics.

    \item We show that the recovery results of \citet{margossian_variational_2025, margossian_generalized_2025} in location--scale families under even and elliptical symmetry arise as special cases of our general framework.

    \item We apply our theory to directional statistics on the sphere and derive novel guarantees for recovering the axis of symmetry in von Mises--Fisher families under rotational symmetry.
\end{enumerate}

Our results are formulated at the level of \(f\)-divergences, thereby encompassing many of the aforementioned variational objectives, and are agnostic to the particular variational inference scheme used to optimise them \citep{hoffman_stochastic_2013, ranganath_black_2014, kucukelbir_automatic_2017}. Conceptually, our theory identifies the fundamental mechanism by which symmetry propagates through variational inference and provides a modular framework for deriving new statistic recovery guarantees in symmetry settings beyond those considered here. 

\section{Related work}

A large body of work has concentrated on asymptotic guarantees for variational approximations. \citet{wang_frequentist_2019} proved a variational Bernstein--von Mises theorem, establishing frequentist consistency of variational estimators. \citet{pati_statistical_2018} and \citet{yang_alpha_2020} proved convergence rates for point estimates constructed from variational approximations with reverse KL and \(\alpha\)-divergence objectives, while \citet{zhang_convergence_2020} proved convergence rates for entire posteriors in nonparametric settings. \citet{alquier_concentration_2020} similarly derived such rates for tempered posteriors.

A parallel line of work studies non-asymptotic guarantees for variational methods. \citet{han_statistical_2019} derived explicit bounds on the KL divergence between a mean-field variational posterior and a normal distribution centred at the maximum likelihood estimator, and showed that the variational mean matches the MLE up to higher-order terms. Motivated by empirical evidence of VI's ability to recover the mean \citep{mackay_information_2003, giordano_covariances_2018}, subsequent work focused on the quality of variational estimates of the target's statistics. In particular, \citet{katsevich_approximation_2024} derived non-asymptotic bounds on the mean and covariance error in Gaussian VI in terms of total variation.

Most closely related to our work is that of \citet{margossian_variational_2025, margossian_generalized_2025}, who showed that, provided the variational minimiser is unique, VI in location--scale families exactly recovers the mean when the target and variational family exhibit even symmetry, and additionally recovers the covariance up to a multiplicative constant when they exhibit elliptical symmetry. Our work generalises these results by moving beyond recovery guarantees tied to specific symmetry classes, distributional families, and statistics, and instead developing a unified symmetry-based framework for statistic recovery.

\section{A symmetry-based framework for statistic recovery}\label{sec:sym_rec}

Let \((\mc X, \mc B)\) be a measurable space, and let \(\mc P\) denote the set of associated probability measures. VI starts with a target distribution \(\mb P \in \mc P\), and seeks the best approximation to it inside a tractable variational family \(\mc Q \subseteq \mc P\) by minimising a statistical divergence \(\mc D \colon \mc P \times \mc P \to [0,\infty] \):
\begin{equation}\tag{\ensuremath{\star}}\label{eq:vi_problem}
\mc Q^\star \coloneqq \argmin_{\mb Q \in \mc Q} \mc D(\mb P \| \mb Q).
\end{equation}
Throughout this paper, we consider \(f\)-divergences \( \mc{D}_f \), a broad class of discrepancy measures that encompasses many of the objectives used in variational inference, including the usual reverse KL divergence.\footnote{The definition and further background are deferred to Appendix~\ref{apdx:f_div}.} In most practical applications, \(\mb P \notin \mc Q\), meaning that the variational family is not rich enough to contain the true target and the variational problem is \emph{misspecified}. For example, \(\mb P\) may be a complicated multimodal distribution and \(\mc Q\) a simple family of Gaussian distributions. Our goal in this section is therefore not necessarily to recover the full target distribution. Instead, we wish to characterise when, despite misspecification, VI can faithfully recover different \emph{statistics} of the target.

We model a statistic as a partial map \(S \colon \mc P \rightharpoonup \mc Y\), where \(\mc Y\) is an output space. This reflects the fact that many statistics are only well-defined on a natural subclass of \(\mc P\), which we call the \emph{domain} of \(S\) and denote by \(\operatorname{dom}(S) \subseteq \mc P\). For instance, for \(\mc X = \mb R^d\), the mean is a statistic mapping distributions with finite first moment to vectors in \(\mb R^d\), while the covariance is a statistic mapping distributions with finite second moment to positive-semidefinite \(d \times d\) matrices. Since recovery statements are only meaningful when the statistic is well-defined, we assume that \(\mb{P} \in \operatorname{dom}(S)\) and \(\mc{Q} \subseteq \operatorname{dom}(S)\). 

We are particularly interested in statistic recovery guarantees that arise from symmetries, such as coordinate permutations, rotations, or sign flips. Such symmetries emerge from transformations of the sample space itself, and in our setting it is therefore natural to represent them by a group \(\mc G\) of measurable bijections \(g \colon \mc X \to \mc X\) with measurable inverses --- a measure-theoretic generalisation of a change of coordinates --- so that they may be composed and inverted.\footnote{Formally, \(\mc G\) is a subgroup of the group of measurable automorphisms \(\mathrm{Aut}(\mc X, \mc B)\).} Each \(g \in \mc G\) acts on any probability measure \(\pi \in \mc P\) via the pushforward map
\[
g_\# \pi(A) \coloneqq \pi(g^{-1}(A)), \qquad A \in \mc B.
\]
Equivalently, if \(X \sim \pi\), then \(g(X) \sim g_\# \pi\). We say that \(\pi \in \mc P\) is \emph{\(\mc G\)-invariant} if \(g_\# \pi = \pi\) for all \(g \in \mc G\), that is, if applying any symmetry in the group leaves the distribution unchanged. For example, a centred isotropic Gaussian in \(\mb{R}^d\) is invariant under the action of the orthogonal group \(\mathrm{O}(d)\) of real \( d \times d\) orthogonal matrices, since it is left unchanged by rotations and reflections about the origin. Similarly, we say that a collection \(\mc F \subseteq \mc P\) is \emph{\(\mc G\)-stable} if for all \(g \in \mc G\), we have \(g_\# \mc F \subseteq \mc F\), meaning that the pushforward of any distribution in \(\mc F\) also lies in \(\mc F\). For example, the Gaussian family in \(\mb R^d\) is stable under the action of the general linear group \(\mathrm{GL}(d, \mb R)\) of real \(d \times d\) invertible matrices, since every linear transformation of a Gaussian random variable is also a Gaussian random variable. 

It turns out that symmetries of the target distribution can constrain the values that its statistics may take, and these constraints can be inherited by a variational minimiser, thus guaranteeing statistic recovery even when \(\mb P \notin \mc Q\). To make this possible, we impose three structural conditions describing how the target, the variational family, and the statistic must interact with \(\mc G\):

\begin{mdframed}[
  linewidth=0.8pt,
  linecolor=black,
  backgroundcolor=white,
  skipabove=8pt,
  skipbelow=8pt,
  innerleftmargin=8pt,
  innerrightmargin=8pt,
  innertopmargin=6pt,
  innerbottommargin=6pt,
  frametitle={\bfseries Conditions}
]
\begin{enumerate}
    \item[(1)] The target \(\mb P \in \operatorname{dom}(S)\) is \(\mc G\)-invariant.
    
    \item[(2)] The variational family \(\mc Q \subseteq \operatorname{dom}(S)\) is \(\mc G\)-stable.
    
    \item[(3)] The statistic \(S \colon \mc P \rightharpoonup \mc Y\) has \(\mc G\)-stable domain, and for all \(g \in \mc G\) and \(\pi_1, \pi_2 \in \operatorname{dom}(S)\), 
    \[
    S(\pi_1) = S(\pi_2)
    \quad\Longrightarrow\quad
    S(g_\# \pi_1) = S(g_\# \pi_2).
    \]
\end{enumerate}
\end{mdframed}

The final condition requires that if two distributions have the same statistic, then their pushforwards by any transformation in \( \mc{G} \) also have the same statistic. This allows us to study the effect of symmetries of a distribution directly at the level of its statistic values. To this end, fix \(y \in \operatorname{Im}(S)\). By definition, there exists at least one distribution \(\pi \in \operatorname{dom}(S)\) such that \(S(\pi) = y\). Now, for any \(g \in \mc{G}\), consider the quantity
\(S(g_{\#}\pi)\), which is well-defined since \(\operatorname{dom}(S)\) is assumed to be \(\mc{G}\)-stable. At first glance, this may seem to depend on the particular choice of \(\pi\) that generated \(y\), since multiple distributions can share the same statistic. However, Condition (3) ensures precisely that if \(\pi_1\) and \(\pi_2\) both satisfy
\(S(\pi_1) = S(\pi_2) = y\) then \(S(g_{\#}\pi_1) = S(g_{\#}\pi_2)\). We may thus introduce, for every \(g \in \mc{G}\), a well-defined map \(\rho_g: \operatorname{Im}(S) \rightarrow \operatorname{Im}(S)\) given by 
\[
\rho_g(y) \coloneqq S(g_{\#}\pi), \quad \text{for any } \pi \text{ such that } S(\pi) = y.
\]
Since our ultimate goal is statistic recovery, we are naturally led to consider the statistic values pinned down by the symmetries in \( \mc{G} \). We collect these in the fixed set
\[
\Gamma_S(\mc{G}) \coloneqq \left\{ y \in \operatorname{Im}(S) \mid \rho_g(y) = y, \ \forall g \in \mc{G} \right\}.
\]
We are now ready to state our main statistic recovery result. At a high level, Conditions~(1) and~(2) ensure that pushing forward a variational minimiser by any element of \(\mc G\) yields another minimiser. If the minimiser is unique, it must therefore inherit the \(\mc G\)-invariance of the target. Condition~(3) then lifts this symmetry from distributions to statistic values, yielding the desired recovery guarantee. 

\begin{theorem}[Statistic recovery]\label{thm:master_theorem}
Let \(S \colon \mc P \rightharpoonup \mc Y\) be a statistic, \(\mb P \in \operatorname{dom}(S)\) a target distribution, and \(\mc Q \subseteq \operatorname{dom}(S)\) a variational family. Assume the variational problem \eqref{eq:vi_problem} has a unique minimiser, denoted by \(\mb Q^\star\). If Conditions~(1)--(3) are satisfied for some group of symmetries \(\mc G\), then
\[
S(\mb P), S(\mb Q^\star) \in \Gamma_S(\mc G).
\]
\end{theorem}

A proof can be found in Appendix~\ref{apdx:proof_thm}. Theorem~\ref{thm:master_theorem} shows that even when \(\mb P \notin \mc Q\), variational inference in a \( \mc{G} \)-stable family against a \(\mc{G}\)-invariant target forces every symmetry-compatible statistic of the target and a unique variational minimiser to lie in the corresponding fixed set. When the fixed set is a singleton, exact recovery of the statistic is achieved. In general, however, the fixed set may instead be a low-dimensional subset of the codomain, in which case the recovery is only partial. We will see examples of both cases in Section~\ref{sec:prev_res}.

\begin{remark}
In practice, variational inference is carried out over a parameterised variational family \(\mc Q = \{\phi(\theta) \mid \theta \in \Theta\}\), where \(\Theta\) is a set of parameters and \(\phi \colon \Theta \to \mc P\) is a parameterisation map, by minimising \(\mc D_f(\mb P \| \phi(\theta))\) over \(\Theta\). While uniqueness of the minimising distribution need not imply uniqueness of the minimising parameter due to non-identifiability of the parameterisation, the converse does hold. As parameter-level uniqueness is often easier to verify in concrete models, this is the form in which we will use Theorem~\ref{thm:master_theorem} in later sections.
\end{remark}

\section{Location--scale families: guarantees under even and elliptical symmetry}\label{sec:prev_res}

In this section, we show that our framework generalises existing statistic recovery guarantees for VI in location--scale families. In particular, we demonstrate that the recovery results of \citet{margossian_variational_2025, margossian_generalized_2025} under even and elliptical symmetry arise as direct corollaries of Theorem~\ref{thm:master_theorem}, obtained by verifying conditions~(1)--(3) and computing the corresponding fixed set in each setting.\footnote{See Appendix~\ref{apdx:differences} for a comparison between our presentation and that of \citet{margossian_variational_2025, margossian_generalized_2025}.}

We begin by fixing the common setup and notation for this section. Throughout, we work on \(\mc X = \mb R^d\) with \(d \ge 1\), equipped with its Borel \(\sigma\)-algebra \(\mc B\), and denote by \(\mc P\) the associated set of probability measures. For \(n \in \mb N\), we let \(\mc P_n \subseteq \mc P\) denote the set of probability measures on \(\mb R^d\) with finite \(n\)-th moment. We also write \(\mc S_+^d\) and \(\mc S_{++}^d\) for the sets of real \(d \times d\) positive-semidefinite and positive-definite matrices, respectively. In all examples, we consider variational families of location--scale form and reason about the recovery of the mean statistic \(\mu \colon \mc P \rightharpoonup \mb R^d\) with domain \(\operatorname{dom}(\mu) = \mc P_1\), and the covariance statistic \(\Sigma \colon \mc P \rightharpoonup \mc S_+^d\) with domain \(\operatorname{dom}(\Sigma) = \mc P_2\), in different symmetry settings.

\begin{definition}[Location--scale family]
Fix a base distribution \(\mb Q_0 \in \mc P\). The location--scale family generated by \(\mb Q_0\) is the collection of distributions \(
\mc Q \coloneqq \{\mb Q_{\nu,S} \mid \nu \in \mb R^d,\ S \in \mc S_{++}^d\} \) where
\[
\mb Q_{\nu,S} \coloneqq (T_{\nu,S})_\# \mb Q_0, \qquad \text{with} \qquad T_{\nu,S}(x) \coloneqq \nu + S^{1/2} x.
\]
\end{definition}

Equivalently, if \(X_0 \sim \mb Q_0\), then \(\nu + S^{1/2} X_0 \sim \mb Q_{\nu,S}\). We note that since each \(\mb Q_{\nu,S}\) is the pushforward of \(\mb Q_0\) under an invertible affine map, \(\mc Q \subseteq \mc P_n\) if and only if \(\mb Q_0 \in \mc P_n\), meaning that the mean and covariance are well-defined for the whole variational family whenever they are well-defined for the base distribution.

\subsection{Mean recovery under even symmetry}\label{sec:mean_even}

We first consider the setting in which the target and variational family exhibit even symmetry, and show that a unique variational minimiser exactly recovers the mean statistic. 

\begin{definition}[Even symmetry]
A distribution \(\pi \in \mc P\) is \emph{even symmetric} about \(a \in \mb R^d\) if
\[
\pi(A) = \pi(2a - A)
\qquad \text{for all } A \in \mc B,
\]
where \(2a - A \coloneqq \{2a - x \mid x \in A\}\).
\end{definition}

To obtain the desired recovery guarantee, we specialise our framework to the even symmetric setting. Fix \(m \in \mb R^d\) and let \(\mb P \in \mc P_1\) be even symmetric about \(m\). Further, let \(\mc Q \subseteq \mc P_1\) be the location--scale family generated by a base distribution \(\mb Q_0 \in \mc P_1\) that is even symmetric about \(0\).\footnote{This is equivalent to requiring that each \(\mb Q_{\nu,S} \in \mc Q\) is even symmetric about its location parameter \(\nu\).} We apply Theorem~\ref{thm:master_theorem} to the mean statistic \(\mu\) under the symmetry group \(\mc G \coloneqq \{e, r_m\}\), where \(e\) denotes the identity map and \(r_a(x) \coloneqq 2a - x\) for all \(a \in \mb R^d\). We begin by verifying Conditions~(1)--(3):
\begin{enumerate}
    \item[(1)] Trivially, we know that \(e_\# \mb P = \mb P\). Moreover, since \(r_m = r_m^{-1}\), we have for every \(A \in \mc B\),
\[
(r_m)_\# \mb P(A)
=
\mb P(r_m^{-1}(A))
=
\mb P(r_m(A))
=
\mb P(A),
\]
where the last equality follows from the even symmetry of \(\mb P\) about \(m\). Thus \(\mb P\) is \(\mc G\)-invariant.

\item[(2)] For any \((\nu, S) \in \mb R^d \times \mc S_{++}^d\), we have \(e_\# \mb Q_{\nu,S} = \mb Q_{\nu,S} \in \mc Q\). Moreover, it is easy to verify that \(r_m \circ T_{\nu,S} = T_{2m - \nu,S} \circ r_0\), and therefore
\begin{equation}\label{eq:scale_unchanged}
\begin{aligned}
(r_m)_\# \mb Q_{\nu,S}
&=
(r_m \circ T_{\nu,S})_\# \mb Q_0
=
(T_{2m-\nu,S} \circ r_0)_\# \mb Q_0 \\
&=
(T_{2m-\nu,S})_\# \bigl((r_0)_\# \mb Q_0\bigr)
=
\mb Q_{2m-\nu,S} \in \mc Q,
\end{aligned}
\end{equation}
where the last equality follows from the even symmetry of \(\mb Q_0\) about \(0\). Hence, \(\mc Q\) is \(\mc G\)-stable.

\item[(3)] First, the domain \(\operatorname{dom}(\mu) = \mc P_1\) is \(\mc G\)-stable. Indeed, if \(\pi \in \mc P_1\), then \(e_\# \pi = \pi \in \mc P_1\) and, moreover, \((r_m)_\# \pi \in \mc P_1\), since
\[
\int_{\mb R^d} \|x\| \, d((r_m)_\# \pi)(x)
=
\int_{\mb R^d} \|2m - x\| \, d\pi(x)
\le
2\|m\| + \int_{\mb R^d} \|x\| \, d\pi(x)
<
\infty.
\]
Further, for every \(\pi \in \mc P_1\), we can compute \(\mu(g_\# \pi)\) for every \(g \in \mc G\) as
\begin{equation}\label{eq:mean_reflection_action}
\mu(g_\# \pi)
=
\begin{cases}
\mu(\pi), & g = e,\\
2m - \mu(\pi), & g = r_m.
\end{cases}
\end{equation}
As the right-hand side of \eqref{eq:mean_reflection_action} depends only on \(\mu(\pi)\), it follows that if \(\mu(\pi_1) = \mu(\pi_2)\), then \(\mu(g_\# \pi_1) = \mu(g_\# \pi_2)\) for all \(\pi_1, \pi_2 \in \mc P_1\) and all \(g \in \mc G\).
\end{enumerate}

We have thus verified the conditions of Theorem~\ref{thm:master_theorem}. Moreover, by \eqref{eq:mean_reflection_action}, Condition~(3) allows us to define the maps \(\rho_e(x)=x\) and \(\rho_{r_m}(x) = 2m-x\) on \(\operatorname{Im}(\mu) = \mb{R}^d\) and compute the corresponding fixed set
\[
\Gamma_\mu(\mc G)
=
\{x \in \mb R^d \mid \rho_g(x) = x,\ \forall g \in \mc G\}
=
\{x \in \mb R^d \mid 2m - x = x\}
=
\{m\}.
\]
Hence, Theorem~\ref{thm:master_theorem} guarantees exact recovery of the mean when specialised to the even-symmetric setting, provided the variational minimiser is unique. We record this result in the following corollary, which coincides with the mean recovery theorem of \citet{margossian_generalized_2025}.

\begin{corollary}[cf. Theorem~10 in \hypersetup{citecolor=black}\citet{margossian_generalized_2025}]\label{cor:mean_rec}
Fix \(m \in \mb R^d\) and let \(\mb P \in \mc P_1\) be even symmetric about \(m\). Let \(\mb Q_0 \in \mc P_1\) be even symmetric about \(0\), and let \(\mc Q\) be the location--scale family it generates. If \(\mc D_f(\mb P \| \mb Q_{\nu,S})\) has a unique minimiser \((\nu^\star, S^\star)\) over \(\mb R^d \times \mc S_{++}^d\), then
\[
\mu(\mb Q_{\nu^\star,S^\star}) = \mu(\mb P) = m.
\]
\end{corollary}

The left panel of Figure~\ref{fig:elliptical_sym} illustrates this result. It is worth noting that in \eqref{eq:scale_unchanged}, the scale matrix remains unchanged under pushforward by elements of \(\mc G\). Corollary~\ref{cor:mean_rec} therefore also applies mutatis mutandis to VI over a location family, i.e., a location--scale family with fixed scale matrix.

\subsection{Mean and covariance recovery under elliptical symmetry}

We next consider the setting where the target and variational family exhibit elliptical symmetry. Here, we show that a unique variational minimiser exactly recovers the mean and further recovers the covariance matrix up to a multiplicative constant. This, in turn, yields exact recovery of the correlation matrix.

\begin{definition}[Elliptical symmetry]
A distribution \(\pi \in \mc{P}\) is \emph{elliptically symmetric} about \(\gamma \in \mb{R}^d\) with scale matrix \(C \in \mc{S}_{++}^d\) if there exists an \(\mathrm{O}(d)\)-invariant distribution \(\pi_0 \in \mc{P}\) such that
\[
\pi = (T_{\gamma,C})_\# \pi_0,
\qquad
T_{\gamma,C}(x)\coloneqq \gamma+C^{1/2}x.
\]
\end{definition}
Here, \(\mathrm{O}(d)\) is the orthogonal group, which consists of real \(d \times d\) orthogonal matrices. Equivalently, if \(X_0 \sim \pi_0\), then \(\gamma+C^{1/2}X_0 \sim \pi\). We proceed, as before, by specialising Theorem~\ref{thm:master_theorem} to this new setting. Let \(\mb{P} \in \mc{P}\) be elliptically symmetric about \(m \in \mb{R}^d\) with scale matrix \(M \in \mc{S}_{++}^d\), and let \(\mc{Q} \subseteq \mc{P}\) be the location--scale family generated by an \(\mathrm{O}(d)\)-invariant base distribution \(\mb{Q}_0 \in \mc{P}\).\footnote{This is equivalent to requiring that each \(\mb{Q}_{\nu, S} \in \mc{Q}\) is elliptically symmetric about \(\nu\) with scale matrix \(S\).}

\paragraph{Mean recovery.} Observe that a distribution that is elliptically symmetric about \(\gamma \in \mb{R}^d\) is also even symmetric about \(\gamma\). Indeed, if \( \pi=(T_{\gamma,C})_\#\pi_0 \) for some \(\mathrm{O}(d)\)-invariant \(\pi_0\), then for every \(A \in \mc{B}\)
\[
\pi(2\gamma-A) = \pi_0(T_{\gamma,C}^{-1}(2\gamma-A)) = \pi_0(-T_{\gamma,C}^{-1}(A)) = \pi_0(T_{\gamma,C}^{-1}(A)) = \pi(A),
\]
where the penultimate step follows from the fact that \(\mathrm{O}(d)\)-invariance of \(\pi_0\) implies \((-I_d)_\# \pi_0 = \pi_0\). Therefore, if \(\mb{P}, \mb{Q}_0 \in \mc{P}_1\), mean recovery is immediate by Corollary~\ref{cor:mean_rec}.

\paragraph{Covariance recovery.}
To reason about covariance recovery, we assume that \(\mb{P}, \mb{Q}_0 \in \mc{P}_2\) and consider the symmetry group of ellipsoids centred at \(m\) with shape matrix \(M\), defined as
\[
\mc{G}\coloneqq \{g_R\mid R\in \mathrm{O}(d)\},
\qquad
g_R(x)\coloneqq m+A_R(x-m),
\qquad
A_R\coloneqq M^{1/2}RM^{-1/2}.
\]
We begin by verifying Conditions~(1)--(3) under the action of this group:

\begin{enumerate}
\item[\((1)\)]
By elliptical symmetry, there exists an \(\mathrm{O}(d)\)-invariant \(\mb{P}_0 \in \mc{P}_2\) such that \(\mb{P}=(T_{m,M})_\#\mb{P}_0\). It is easy to check that \(g_R\circ T_{m,M}=T_{m,M}\circ R\) for every \(R \in \mathrm{O}(d)\), and therefore
\[
(g_R)_\# \mb{P}
=
(g_R\circ T_{m,M})_\#\mb{P}_0
=
(T_{m,M}\circ R)_\#\mb{P}_0
=
(T_{m,M})_\#(R_\#\mb{P}_0)
=
\mb{P}.
\]

\item[\((2)\)] Let \((\nu,S)\in \mb{R}^d\times \mc{S}_{++}^d\) and \(R\in \mathrm{O}(d)\). Define
\[
\nu' \coloneqq m+A_R(\nu-m),
\qquad
S' \coloneqq A_RSA_R^\top,
\qquad
U \coloneqq (S')^{-1/2}A_RS^{1/2}.
\]
Then \(U\in \mathrm{O}(d)\) because \(UU^\top = I_d\), and since \(\mb{Q}_0\) is \(\mathrm{O}(d)\)-invariant, we have \(U_\# \mb{Q}_0=\mb{Q}_0\). Moreover, we have that \(g_R\circ T_{\nu,S}
=
T_{\nu',S'}\circ U\) and therefore
\[
(g_R)_\#\mb{Q}_{\nu,S}
=
(g_R\circ T_{\nu,S})_\#\mb{Q}_0
=
(T_{\nu',S'}\circ U)_\#\mb{Q}_0
=
(T_{\nu',S'})_\#(U_\#\mb{Q}_0)
=
\mb{Q}_{\nu',S'}
\in \mc{Q}.
\]

\item[\((3)\)] First, the domain \(\mc{P}_2\) is \(\mc{G}\)-stable. Indeed, for all \(x \in \mb{R}^d\)
\begin{equation}\label{eq:operator_norm}
\| g_R(x)\| \le \| m \| + \|A_R \|_{\mathrm{op}} (\|x \| + \| m \|),
\end{equation}
and squaring \eqref{eq:operator_norm}, together with the identity \((a+b)^2 \le 2a^2 + 2b^2\), yields that for any \(\pi \in \mc{P}_2\)
\[
\int_{\mb{R}^d} \| x \|^2 \, d((g_R)_\#\pi)(x)
\le
2\|A_R\|_{\mathrm{op}}^2 \int_{\mb{R}^d}\|x\|^2\,d\pi(x)
+
2(1+\|A_R\|_{\mathrm{op}})^2\|m\|^2
<\infty,
\]
since \(\| A_R \|_{\mathrm{op}} < \infty\) by the submultiplicativity of the operator norm. Moreover, for every \(\pi \in \mc{P}_2\) and every \(R \in \mathrm{O}(d)\), we have \( \mu((g_R)_\# \pi) = g_R(\mu(\pi)) \) and so
\begin{equation}\label{eq:cov_induced_action}
\begin{aligned}
\Sigma((g_R)_\#\pi)
&=
\int_{\mb{R}^d}
\bigl(g_R(x)-g_R(\mu(\pi))\bigr)
\bigl(g_R(x)-g_R(\mu(\pi))\bigr)^\top
\,d\pi(x) \\
&=
A_R
\left(
\int_{\mb{R}^d}
(x-\mu(\pi))(x-\mu(\pi))^\top
\,d\pi(x)
\right)
A_R^\top \\
&=
A_R\,\Sigma(\pi)\,A_R^\top,
\end{aligned}
\end{equation}
meaning that if \(\Sigma(\pi_1)=\Sigma(\pi_2)\), then
\(\Sigma((g_R)_\#\pi_1)=\Sigma((g_R)_\#\pi_2)\) for all \(R \in \mathrm{O}(d)\).
\end{enumerate}

By \eqref{eq:cov_induced_action}, Condition~(3) allows us to define for every \(g_R \in \mc{G}\) the map \(\rho_{g_R}(V) = A_R V A_R^\top\) on \(\operatorname{Im}(\Sigma) = \mc{S}_+^d\) and find the positive-semidefinite matrices satisfying \(\rho_{g_R}(V) = V\) for all \(g_R \in \mc{G}\): 
\[
\Gamma_\Sigma(\mc{G})
= \{V\in \mc{S}_+^d \mid A_RVA_R^\top=V,\ \forall R\in \mathrm{O}(d)\}
=
\{\lambda M \mid \lambda\ge 0\}.
\]
The last equality follows from the fact that a matrix \(V \in \mc{S}_+^d\) lies in the fixed set if and only if \(M^{-1/2}VM^{-1/2}\) is invariant under conjugation by every \(R \in \mathrm{O}(d)\), since \(A_R = M^{1/2}RM^{-1/2}\), which in turn implies that \(M^{-1/2}VM^{-1/2} = \lambda I_d\) for some \(\lambda \ge 0\). It follows that if \(\mb{P}\) is not a Dirac measure, then \( \Sigma(\mb{P}) = \alpha M\) for some \( \alpha > 0\). Similarly, if \(\mb{Q}_0\) is not a Dirac measure and \((\nu^\star, S^\star)\) is a unique minimiser of the variational objective over \(\mb{R}^d \times \mc{S}_{++}^d\), then \(\Sigma(\mb{Q}_{\nu^\star, S^\star}) = bM\) for some \(b > 0\). Therefore, outside the trivial point-mass case, a unique variational minimiser recovers the covariance up to a positive multiplicative constant. 

\paragraph{Correlation recovery.} We conclude this section by noting an immediate implication of covariance recovery for the correlation statistic. To that end, let \(\mc{P}_2^\circ \subseteq \mc{P}_2\) denote the set of distributions with finite second moment and positive marginal variances, and assume \(\mb{P}, \mb{Q}_0 \in \mc{P}_2^\circ\). Since \(\mb{Q}_0\) is \(\mathrm{O}(d)\)-invariant, its covariance is equal to a positive scalar multiple of the identity, which guarantees that \(\mc{Q} \subseteq \mc{P}_2^\circ\). From covariance recovery, we know that \(\Sigma(\mb{P}) = aM\) for some \(a > 0\), and \(\Sigma(\mb{Q}_{\nu^\star, S^\star}) = bM\) for some \(b > 0\). To remove this scale ambiguity, we may define for every \(\pi \in \mc{P}_2^\circ\) the correlation statistic
\[
\rho(\pi) \coloneqq D(\pi)^{-1/2}\Sigma(\pi)D(\pi)^{-1/2}, \qquad D(\pi) \coloneqq \operatorname{diag}(\Sigma(\pi)).
\]

Then \(D(\mb{P}) = a\,\operatorname{diag}(M)\) and \(D(\mb{Q}_{\nu^\star, S^\star}) = b\,\operatorname{diag}(M)\), and thus
\[
\rho(\mb{P})
=
\operatorname{diag}(M)^{-1/2}M\operatorname{diag}(M)^{-1/2}
=
\rho(\mb{Q}_{\nu^\star, S^\star}).
\]
Therefore, the correlation is recovered exactly by a unique minimiser. We summarise our discussion on statistic recovery under elliptical symmetry in the following corollary, which coincides with Theorem~11 of \citet{margossian_generalized_2025}. An example is provided in the right panel of Figure~\ref{fig:elliptical_sym}.

\begin{figure}[t]
    \centering
    \includegraphics[width=\linewidth]{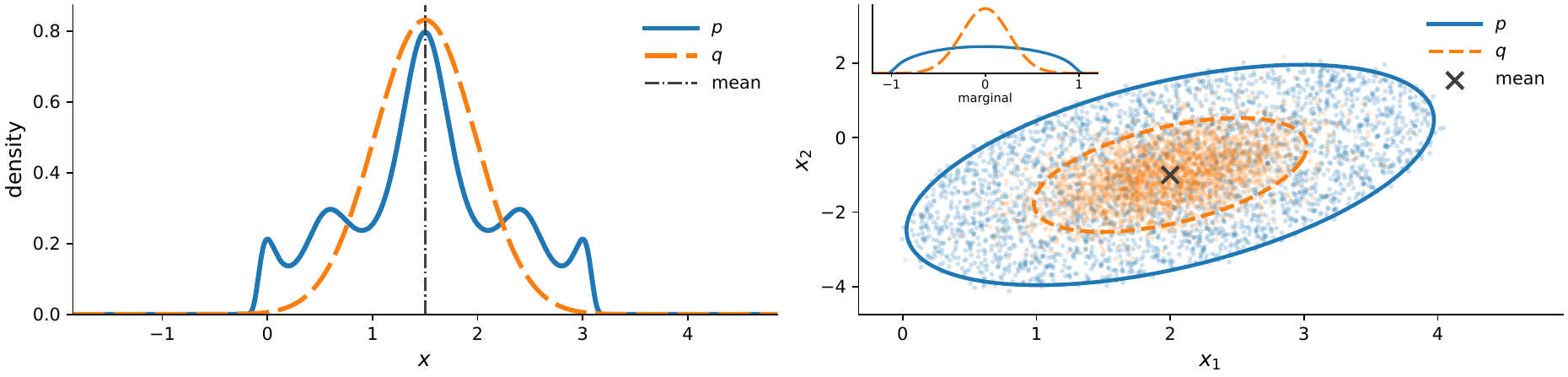}
    \caption{Unique variational minimisers \(q\) from a Gaussian family exactly recover symmetry-determined statistics of highly non-Gaussian targets \(p\). Left: under even symmetry, the mean is recovered exactly as per Corollary~\ref{cor:mean_rec}. Right: under elliptical symmetry, the mean and correlation are recovered exactly, whereas the covariance is recovered only up to scale, as per Corollary~\ref{cor:elliptical_sym}.}
    \label{fig:elliptical_sym}
\end{figure}

\begin{corollary}[cf. Theorem~11 in \hypersetup{citecolor=black}\citet{margossian_generalized_2025}]\label{cor:elliptical_sym}
Fix \(m\in \mb{R}^d\) and \(M\in \mc{S}_{++}^d\), and let \(\mb{P}\) be elliptically symmetric about \(m\) with scale matrix \(M\). Let \(\mb{Q}_0\) be an \(\mathrm{O}(d)\)-invariant base distribution, and let \(\mc{Q}\) be the location--scale family it generates. If \(\mc{D}_f(\mb{P} \| \mb{Q}_{\nu, S})\) has a unique minimiser \((\nu^\star,S^\star)\) over \(\mb{R}^d\times \mc{S}_{++}^d\), then
\begin{enumerate}
\item[(i)] if \(\mb{P},\mb{Q}_0\in \mc{P}_1\), then
\[
\mu(\mb{Q}_{\nu^\star,S^\star})=\mu(\mb{P})=m,
\]
\item[(ii)] if \(\mb{P},\mb{Q}_0\in \mc{P}_2\) and are not Dirac measures then
\[
\Sigma(\mb{Q}_{\nu^\star,S^\star}), \Sigma(\mb{P}) \in \{ \lambda M \mid \lambda > 0 \},
\]
\item[(iii)] if in addition \(\mb{P}, \mb{Q}_0 \in \mc{P}_2^\circ\), then
\[
\rho(\mb{Q}_{\nu^\star,S^\star}) = \rho(\mb{P}) = \operatorname{diag}(M)^{-1/2}M\operatorname{diag}(M)^{-1/2}.
\]
\end{enumerate}
\end{corollary}

\begin{remark}
Note that although the conditions of Theorem~\ref{thm:master_theorem} are sufficient for statistic recovery, they are not always necessary. Indeed, correlation is recovered in Corollary~\ref{cor:elliptical_sym} through the covariance, even though it does not itself satisfy Condition~(3) for \( d \ge 2 \); see Appendix~\ref{apdx:corr_not_compatible} for a discussion.
\end{remark}

\section{von Mises--Fisher family: new guarantees under rotational symmetry}\label{sec:sphere}

So far, we have demonstrated that previously known statistic recovery results arise as special instances of our general framework. We now use this framework to obtain a new recovery guarantee for variational inference on the sphere. Specifically, we show that when the target exhibits rotational symmetry, variational inference in a von Mises--Fisher family can exactly recover the axis of symmetry. Unlike the location--scale examples, where the relevant symmetry, variational family, and recovered statistic were known in advance, here we show how our framework can be applied more organically: we start from a symmetry of the target distribution, determine a statistic that is pinned down by it, and purposely select a stable variational family in order to guarantee its recovery. This section therefore not only establishes a new result, but also provides a blueprint illustrating how our framework can be applied to new symmetries and statistics beyond the examples considered here.

Let \(\mc X = \mc S^{d-1}\) be the unit sphere in \(\mb R^d\), \(d \ge 3\), equipped with its Borel \(\sigma\)-algebra \(\mc B\), and let \(\mc P\) denote the associated set of probability measures. To ease the presentation, we work with distributions that are absolutely continuous with respect to the uniform probability measure \(\sigma\) on the unit sphere, which we collect inside the set \(\mc P_{\sigma} \subseteq \mc P\). For any \(\pi \in \mc P_{\sigma}\), we denote its density with respect to \(\sigma\) by \(f_{\pi}\). 

\begin{definition}[Rotational symmetry]
A distribution \(\pi \in \mc P_{\sigma}\) is \emph{rotationally symmetric} about axis direction \(w \in \mc S^{d-1}\) if there exists a measurable function \(\psi \colon [-1,1] \to \mb R_{\ge 0}\) such that
\[
f_{\pi}(x) = \psi(w^\top x), \qquad \text{for } \sigma\text{-a.e. } x \in \mc S^{d-1}.
\]
We call \(\psi\) the \emph{axial profile} of \(\pi\). If \(\psi\) is constant, then \(\pi\) is the uniform distribution on \(\mc S^{d-1}\).
\end{definition}

Suppose we are given a non-uniform target \(\mb P \in \mc P_{\sigma}\) that is rotationally symmetric about \(u \in \mc S^{d-1}\). Observe that the level sets of the map \(x \mapsto u^\top x\) are precisely the parallels orthogonal to \(u\). Thus, a rotationally symmetric distribution may vary only with the polar angle away from \(u\). In particular, it is invariant under the group \(\mc G \coloneqq \{ g_R \mid Ru = u,\ R \in \mathrm{SO}(d) \}\), where \(\mathrm{SO}(d)\) is the special orthogonal group of orthogonal \(d \times d\) matrices with determinant \(1\), and \(g_R(x) = Rx\) for all \(x \in \mc S^{d-1}\). Indeed, \((g_R)_\# \mb{P} = \mb{P}\) for all \(g_R \in \mc G\), since by rotation invariance of \(\sigma\) we have for \(\sigma\)-a.e. \(x \in \mc S^{d-1}\),
\[
f_{(g_R)_\# \mb P}(x)=f_{\mb P}(R^{-1}x)=\psi(u^\top R^{-1}x)=\psi((Ru)^\top x)=\psi(u^\top x)=f_{\mb P}(x).
\]
Given this target symmetry, we now ask which statistic, if any, it pins down. A natural first candidate is the axis direction. However, if \(\pi \in \mc{P}_{\sigma}\) is rotationally symmetric about \(w\) with axial profile \(\psi\), it is also rotationally symmetric about \(-w\) with axial profile \(\tilde{\psi}(t) \coloneqq \psi(-t)\). Thus, the most the symmetry structure can uniquely determine is a one-dimensional subspace of \(\mb R^d\). More precisely, if \(\pi \in \mc P_{\sigma}\) is non-uniform and rotationally symmetric about both \(w_1, w_2 \in \mc S^{d-1}\), then \(\operatorname{span}(w_1) = \operatorname{span}(w_2)\); proof in Appendix~\ref{apdx:unique_axis}. We therefore consider the space of lines through the origin in \( \mb{R}^d \) denoted by \( \mb{R}P^{d-1} \), let \(\mc P_{\mc A} \subseteq \mc P_{\sigma}\) be the set of all non-uniform rotationally symmetric distributions, and define the axis statistic \(\mc A \colon \mc P \rightharpoonup \mb{R}P^{d-1} \) on \(\operatorname{dom}(\mc A) = \mc P_{\mc A}\) by \(\mc A(\pi) = \operatorname{span}(w)\), where \(w\) is any axis direction of rotational symmetry of \(\pi\). \(\mc{P}_{\mc{A}}\) is \(\mc G\)-stable, since for any \(\pi \in \mc{P}_{\mc{A}}\) rotationally symmetric about \(w\) with axial profile \(\psi\) and any \(g_R \in \mc G \), 
\begin{equation*}
f_{(g_R)_\# \pi}(x) = f_{\pi}\left(R^{-1} x\right) = \psi(w^\top R^{-1}x) = \psi\left((Rw)^\top x\right), \qquad \text{for } \sigma\text{-a.e. } x \in \mc S^{d-1},
\end{equation*}
meaning \( (g_R)_\# \pi \) is rotationally symmetric about \(Rw\) with profile \( \psi \) and thus \((g_R)_\# \pi \in \mc P_{\mc A}\). Hence,
\begin{equation}\label{eq:axis_pushforward}
\mc A((g_R)_\# \pi) = \operatorname{span}(Rw) = R \operatorname{span}(w) = R \mc A(\pi),
\end{equation}
which implies that if \(\mc A(\pi_1) = \mc A(\pi_2)\), then \(\mc A((g_R)_\# \pi_1) = \mc A((g_R)_\# \pi_2)\) for all \(g_R \in \mc G \). By \eqref{eq:axis_pushforward}, we can define for every \(g_R \in \mc{G}\) the map \(\rho_{g_R}(L) \coloneqq RL\) on \(\operatorname{Im}(\mc{A}) = \mb{R}P^{d-1}\) to obtain the fixed set
\[
\Gamma_{\mc A}(\mc G) = \{ L \in \mb{R}P^{d-1} \mid \rho_{g_R}(L) = L,\ \forall g_R \in \mc G \} = \{ \operatorname{span}(u)\}.
\]
By construction, \(\mb P \in \mc P_{\mc A}\). We thus know that choosing a \(\mc{G}\)-stable family will guarantee the recovery of the axis statistic by a unique minimiser. One such family is the \emph{von Mises--Fisher family}. Each of its members is rotationally symmetric about a mean direction \(\nu \in \mc S^{d-1}\), and indexed by a concentration parameter \(\kappa > 0\) controlling how quickly the density decays as the angle from \(\nu\) increases.
\begin{definition}[von Mises--Fisher family]
A von Mises--Fisher (vMF) family \(\mc Q \subseteq \mc P_{\sigma}\) is a two-parameter family \(\{ \mb Q_{\nu, \kappa} \mid \nu \in \mc S^{d-1}, \kappa \in \mb R_{> 0}\}\), where for a normalising constant \(c_d(\kappa) > 0\),
\[
f_{\mb Q_{\nu, \kappa}}(x)
=
c_d(\kappa)\exp(\kappa \nu^\top x), \qquad \text{for } \sigma\text{-a.e. } x \in \mc S^{d-1}.
\]
\end{definition}

Indeed, the von Mises--Fisher family is \(\mc G\)-stable. For all \((\nu, \kappa) \in \mc S^{d-1} \times \mb R_{> 0}\) and all \(g_R \in \mc G\), \((g_R)_\# \mb Q_{\nu, \kappa} = \mb Q_{R\nu, \kappa} \in \mc Q\) since by rotation invariance of \(\sigma\), we have that for \(\sigma\)-a.e. \(x \in \mc S^{d-1}\)
\begin{equation}\label{eq:sphere_pdf}
f_{(g_R)_\#\mb Q_{\nu, \kappa}}(x)
=
c_d(\kappa)\exp\!\big(\kappa \nu^\top R^{-1}x\big)
=
c_d(\kappa)\exp\!\big(\kappa (R\nu)^\top x\big) = f_{\mb Q_{R\nu, \kappa}}(x),
\end{equation}
and since \(\kappa > 0\), it follows that \(\mc{Q} \subseteq \mc{P}_{\mc{A}}\). We thus obtain the following new recovery guarantee.
\begin{figure}[t]
\centering
\includegraphics[width=\textwidth,height=0.166\textheight,keepaspectratio]{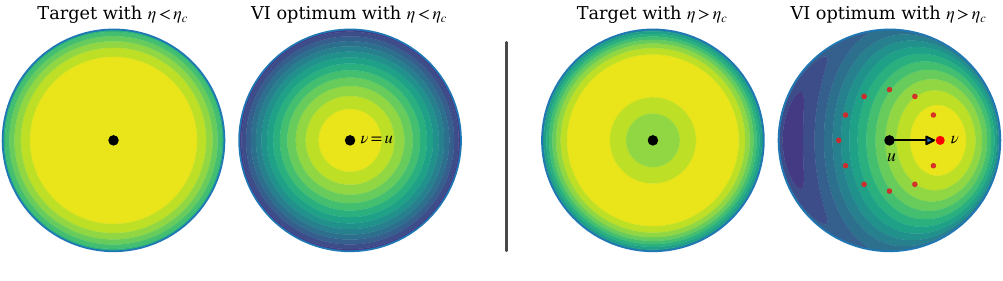}
\caption{Log-density contours of the target and variational posterior are shown using a Lambert azimuthal equal-area projection centred at the true direction \(u\). Left: \(\eta < \eta_c\) and the reverse KL has a unique minimiser at \(\nu = u\), recovering the symmetry axis. Right: \(\eta > \eta_c\) and the minimisers form a latitude circle. None of them, including the one found during optimisation, recovers the statistic.}
\label{fig:sphere_vi}
\end{figure}
\begin{corollary}\label{cor:sphere_cor}
Let \(\mb P \in \mc P_{\mc A}\) be rotationally symmetric about \(u \in \mc S^{d-1}\), and let \(\mc Q \subseteq \mc P_{\mc A}\) be the von Mises--Fisher family. If \(\mc D_f(\mb P \| \mb Q_{\nu, \kappa})\) has a unique minimiser \((\nu^\star, \kappa^\star)\) over \(\mc S^{d-1} \times \mb R_{> 0}\), then
\[
\mc A(\mb Q_{\nu^\star, \kappa^\star}) = \mc A(\mb P) = \operatorname{span}(u).
\]
\end{corollary}
Analogously to location and location--scale families, it is natural to consider a fixed-concentration von Mises--Fisher family, namely, the one-parameter subfamily of distributions that share the same concentration parameter \(\kappa_0 > 0\). Since the pushforward action leaves the concentration parameter unchanged as per \eqref{eq:sphere_pdf}, Corollary~\ref{cor:sphere_cor} equally applies to fixed-concentration vMF families.

\begin{remark}
As with all corollaries of Theorem~\ref{thm:master_theorem}, uniqueness of the minimiser is necessary for the recovery to be guaranteed. In what follows, we highlight this with a concrete example.
\end{remark}

Fix \(\lambda \ne 0\) and \(u \in \mc S^{d-1}\) and let \((\mb P_{\lambda, \eta})_{\eta > 0} \subseteq \mc P_{\mc A}\) be a one-parameter family of target distributions such that for each \(\eta > 0\), \(\mb P_{\lambda, \eta}\) is rotationally symmetric about \(u\) with axial profile
\[
\psi_{\lambda, \eta}(t) \coloneqq Z^{-1}_{\lambda, \eta} \exp\left( \lambda t - \eta t^2 \right), \qquad t \in [-1, 1],
\]
where \(Z_{\lambda, \eta} > 0\) is a normalising constant. Let \(\mc Q_{\kappa_0}\) be a fixed-concentration von Mises--Fisher family for some \(\kappa_0 > 0\) and consider the reverse KL objective \(\mc D_{\mathrm{KL}}(\mb Q_{\nu, \kappa_0} \| \mb P_{\lambda, \eta})\). Since \(\eta > 0\), \(\mb P_{\lambda, \eta}\) is not a vMF distribution and the model is misspecified. Despite that, there exists a critical threshold \(\eta_c(d, \lambda, \kappa_0) > 0\) such that for all \(0 < \eta \le \eta_c\), the objective \(\mc D_{\mathrm{KL}}(\mb Q_{\nu, \kappa_0} \| \mb P_{\lambda, \eta})\) has a unique minimiser \(\nu^\star\) over \(\mc S^{d-1}\), and exact recovery holds with \(\mc A(\mb Q_{\nu^\star,\kappa_0}) = \mc A(\mb P_{\lambda,\eta})\). By contrast, for every \(\eta > \eta_c\), the minimiser set is non-singleton and no minimiser recovers the true axis statistic. This behaviour is visualised for \(d=3\) in Figure~\ref{fig:sphere_vi}, with an explicit construction given in Appendix~\ref{apdx:sphere_construction}.

\section{Discussion}
In this work, we have developed a general symmetry-based theory of statistic recovery in variational inference, and have used it both to recover existing results and to derive novel guarantees. With our framework providing a modular blueprint for establishing statistic recovery, we see significant opportunities to apply it across different domains, symmetry classes, and variational families. At the same time, we identify at least two important directions for future work. First, while our theory provides sufficient conditions for statistic recovery, we have shown that these conditions are not always necessary. Deriving necessary and sufficient conditions would therefore yield a complete characterisation of statistic recovery under symmetry in variational problems. Second, it would be valuable to extend the theory to settings where the target is only approximately symmetric. Such an extension could unlock new guarantees for a broad class of inference problems that exhibit substantial structure but do not exactly satisfy our assumptions --- a direction we leave for future work.

\begin{ack}
    D. Marks acknowledges funding from a G-Research graduate scholarship. D. Paccagnan was partially supported by the EPSRC grant EP/Y001001/1, funded by the International Science Partnerships Fund (ISPF) and UKRI; and by an Imperial-MIT seed fund. 
\end{ack}

\bibliographystyle{plainnat}
\bibliography{references}

\newpage

\appendix

\section{Additional material for Section~\blackref{sec:sym_rec}}\label{apdx:sym_rec}

\subsection{Overview of \(f\)-divergences}\label{apdx:f_div}
In this appendix, we give a formal definition of \(f\)-divergences, which underpin our formulation of the variational inference problem throughout the paper. Let \((\mc X, \mc B)\) be a measurable space, and denote by \(\mc P\) the set of associated probability measures. Let \(\mb P, \mb Q \in \mc P\) be two such probability measures, and let \(\mu\) be any \( \sigma \)-finite measure on \((\mc X, \mc B)\) such that \(\mb P \ll \mu\) and \(\mb Q \ll \mu\), e.g. \(\mu = \mb P + \mb Q\). Write \(p = d\mb P / d\mu\) and \(q = d\mb Q / d\mu\) for the corresponding densities. Let \(f \colon \mb R_{> 0} \to \mb R\) be a convex function with \(f(1) = 0\), and define the extensions:
\[
f(0) \coloneqq \lim_{t \downarrow 0} f(t), \qquad f'(\infty) \coloneqq \lim_{t \downarrow 0} t f\left(\frac{1}{t}\right).
\]
The \(f\)-divergence generated by \(f\), denoted \(\mc D_f \colon \mc P \times \mc P \to [0, \infty]\), is defined by 
\[
\mc D_f(\mb P \| \mb Q) \coloneqq \int_{\{q > 0\}} q f\left(\frac{p}{q}\right)\, d\mu + f'(\infty)\, \mb P(\{q = 0\}),
\]
with the convention that \(0 \cdot \infty = 0\). This definition follows \citet[Def. 7.1]{polyanskiy_information_2025}.\footnote{Note that \citet{polyanskiy_information_2025} work with the standing assumption that \( (\mc{X}, \mc{B}) \) is a standard Borel space. However, their \(f\)-divergence definition extends verbatim to the case of an arbitrary measurable space, such as the one considered here.} A number of popular divergences can be recovered as special cases of \(f\)-divergences with different generators. These include the Kullback--Leibler (KL) divergence with \(f(t) = t \log t\), the \(\chi^2\)-divergence with \(f(t) = (t-1)^2\), the total variation distance with \(f(t) = \tfrac{1}{2}|t-1|\), and the squared Hellinger distance with \(f(t) = (1-\sqrt{t})^2\). For any \(f\)-divergence, swapping the order of the arguments again yields an \( f \)-divergence, now generated by \(\tilde{f}(t) = t f(1/t)\). That is, \( \mc{D}_f( \mb{Q} \| \mb{P}) = \mc{D}_{\tilde{f}}(\mb{P} \| \mb{Q}) \). For example, since the forward KL divergence is generated by \(f(t) = t \log t\), the reverse KL divergence is generated by \(\tilde{f}(t) = t f(1/t) = - \log t\). 

For a detailed overview of \(f\)-divergences, see \citet{sason_f_2016}.

\subsection{Proof of Theorem~\blackref{thm:master_theorem}}\label{apdx:proof_thm}

In this appendix, we provide a proof of Theorem~\ref{thm:master_theorem}. The proof is divided into three parts. In the first part, we show that \(f\)-divergences are invariant under simultaneous pushforward by measurable automorphisms. In the second part, we use this fact to show that Conditions~(1) and~(2) alone suffice to guarantee that a unique variational minimiser inherits the \(\mc{G}\)-invariance of the target. In the third part, we use the maps \(\rho_g\) on statistic values introduced in Section~\ref{sec:sym_rec}, whose well-definedness is guaranteed by Condition~(3), and show that when all conditions hold, the statistic of both the target and a unique variational minimiser is constrained inside the fixed set \(\Gamma_S(\mc G)\).
\paragraph{Setup.} We formally recall the setup of Section~\ref{sec:sym_rec}. Let \((\mc X,\mc B)\) be a measurable space, and denote by \(\mc P\) the set of probability measures on \((\mc X,\mc B)\). Let \(\mathrm{Aut}(\mc X,\mc B)\) be the group of measurable automorphisms, acting on \(\mc P\) by pushforward, and fix a subgroup \(\mc G \le \mathrm{Aut}(\mc X,\mc B)\). Throughout this section, we consider the variational inference problem \eqref{eq:vi_problem} with \(f\)-divergences
\[
\mc{Q}^\star \coloneqq \argmin_{\mb{Q} \in \mc{Q}} \mc{D}_f(\mb{P} \| \mb{Q}).
\]
For the first two parts of the proof, which concern only the variational problem and the action of \(\mc G\) on distributions, we work with a general \(\mb{P} \in \mc P\) and \(\mc{Q} \subseteq \mc P\). For the third part, we additionally consider a statistic \(S \colon \mc P \rightharpoonup \mc Y\), with domain \(\operatorname{dom}(S)\), and ultimately return to the assumption that \(\mb{P} \in \operatorname{dom}(S)\) and \(\mc{Q} \subseteq \operatorname{dom}(S)\) for the last step of the proof of Theorem~\ref{thm:master_theorem}. Whenever Conditions~(1)--(3) are invoked below, we mean the three conditions stated in Section~\ref{sec:sym_rec}.

We begin with the first part of the proof, and show that \(f\)-divergences are invariant under simultaneous pushforward by measurable automorphisms.

\begin{lemma}[\(f\)-divergence invariance]\label{lemma:dpi}
For any \(\mb{P}, \mb{Q} \in \mc P\) and \(g \in \mathrm{Aut}(\mc X, \mc B)\),
\[
\mc{D}_f\!\left(g_\# \mb{P}\,\big\|\,g_\# \mb{Q}\right) = \mc{D}_f\!\left(\mb{P}\,\big\|\,\mb{Q}\right).
\]
\end{lemma}

\begin{proof}
    Let \(\mu\) be a \(\sigma\)-finite common dominating measure of \(\mb{P}\) and \(\mb{Q}\), write \(p = d\mb{P} / d\mu\) and \(q = d\mb{Q} / d\mu\) for the corresponding densities, and define \(\tilde{\mu} = g_\# \mu\). It is easy to verify that \(\tilde{\mu}\) is \(\sigma\)-finite and that \(g_\# \mb{P}, g_\# \mb{Q} \ll \tilde{\mu}\).
Define \(\tilde p := p\circ g^{-1}\) and \(\tilde q := q\circ g^{-1}\). Then for every \(A\in\mc B\),
\[
\int_A \tilde p\,d\tilde\mu
= \int_{\mc X} \mathbf 1_A\, (p\circ g^{-1})\, d\tilde\mu
= \int_{\mc X} (\mathbf 1_A\circ g)\, p\, d\mu
= \int_{g^{-1}(A)} p\, d\mu
= \mb P(g^{-1}(A))
= (g_\#\mb P)(A),
\]
meaning \(\tilde p = d(g_\#\mb P)/d\tilde\mu\), and similarly \(\tilde q = d(g_\#\mb Q)/d\tilde\mu\). We first reduce to the case \(f\ge 0\). For any \(a\in\mb R\), define
\[
h_a(t):=f(t)+a(t-1), \qquad t>0.
\]
Then \(h_a\) is convex, \(h_a(1)=0\), and its extensions satisfy
\[
h_a(0)=f(0)-a,\qquad h_a'(\infty)=f'(\infty)+a.
\]
Moreover,
\begin{align*}
\mc D_{h_a}(\mb P\|\mb Q)
&= \int_{\{q>0\}} q\Bigl(f(p/q)+a(p/q-1)\Bigr)\,d\mu
   + \bigl(f'(\infty)+a\bigr)\mb P(\{q=0\}) \\
&= \mc D_f(\mb P\|\mb Q)
   + a\left[\int_{\{q>0\}}(p-q)\,d\mu + \mb P(\{q=0\})\right].
\end{align*}
Now, by linearity of the integral
\[
\int_{\{q>0\}}(p-q)\,d\mu + \mb P(\{q=0\})
= \mb P(\{q>0\}) - \mb Q(\{q>0\}) + \mb P(\{q=0\}),
\]
and since \(\mb{Q}(\{ q = 0 \}) = 0\) we have \(\mb Q(\{q>0\})=1\). Hence
\[
\int_{\{q>0\}}(p-q)\,d\mu + \mb P(\{q=0\}) = 1-1 = 0,
\]
and therefore
\[
\mc D_{h_a}(\mb P\|\mb Q)=\mc D_f(\mb P\|\mb Q).
\]
Now choose \(\alpha\in\partial f(1)\). By convexity \citep{niculescu_convex_2018},
\[
f(t)\ge f(1)+\alpha(t-1)=\alpha(t-1),\qquad t>0,
\]
so \(h_{-\alpha}(t)=f(t)-\alpha(t-1)\ge 0\). Therefore, replacing \(f\) by the equivalent generator \(h_{-\alpha}\) if necessary, we may assume from the outset that \(f\ge 0\). With this reduction in hand, define
\[
\Phi(x):=
\begin{cases}
\tilde q(x)\, f\!\bigl(\tilde p(x)/\tilde q(x)\bigr), & \tilde q(x)>0,\\
0, & \tilde q(x)=0.
\end{cases}
\]
Since \(f\ge 0\), \(\Phi\) is nonnegative and measurable, so a change of variables yields \citep{tao_introduction_2011}
\[
\int_{\mc X}\Phi\,d\tilde\mu
= \int_{\mc X}\Phi\circ g\,d\mu.
\]
Using \(\tilde p\circ g = p\) and \(\tilde q\circ g = q\), we get \(
\Phi\circ g
= \mathbf 1_{\{q>0\}}\, q\, f(p/q)\). Therefore,
\[
\int_{\{\tilde q>0\}} \tilde q\, f(\tilde p/\tilde q)\, d\tilde\mu
=
\int_{\{q>0\}} q\, f(p/q)\, d\mu.
\]

For the singular term, since \(\tilde q=q\circ g^{-1}\), we have \(
g^{-1}(\{\tilde q=0\})=\{q=0\}\), and hence
\[
(g_\#\mb P)(\{\tilde q=0\})
= \mb P(g^{-1}(\{\tilde q=0\}))
= \mb P(\{q=0\}).
\]

Combining the two identities yields
\[
\mc D_f(g_\#\mb P\|g_\#\mb Q)
=
\int_{\{q>0\}} q\,f(p/q)\,d\mu
+ f'(\infty)\,\mb P(\{q=0\})
=
\mc D_f(\mb P\|\mb Q).
\]
This proves the claim.
\end{proof}

We now proceed with the second part of the proof, where we isolate the role of Conditions~(1) and~(2), namely the \(\mc G\)-invariance of the target and the \(\mc G\)-stability of the variational family, to show that symmetries of the target propagate to unique variational minimisers.

\begin{lemma}[Symmetry propagation]\label{lem:sym_rec}
Let \(\mc G \le \mathrm{Aut}(\mc X,\mc B)\). Assume \(\mb P \in \mc P\) is \(\mc G\)-invariant, and let \(\mc Q \subseteq \mc P\) be \(\mc G\)-stable. Then:
\begin{enumerate}[label=(\roman*)]
\item \(\mc Q^\star\) is \(\mc G\)-stable,
\item if \(\mc Q^\star = \{\mb Q^\star\}\) is a singleton, then \(\mb Q^\star\) is \(\mc G\)-invariant.
\end{enumerate}
\end{lemma}

\begin{proof}
For any \(g \in \mc G\), the objective is invariant under simultaneous pushforward by Lemma~\ref{lemma:dpi}. Now if \(\mb P\) is \(\mc G\)-invariant, then \(g_\# \mb P = \mb P\) and hence
\begin{equation}\label{eq:orbit-const}
\mc D_f(\mb P \| g_\# \mb Q)
=
\mc D_f(g_\# \mb P \,\|\, g_\# \mb Q)
=
\mc D_f(\mb P \,\|\, \mb Q),
\end{equation}
for all \(\mb Q \in \mc Q\). Thus the variational objective is constant along \(\mc G\)-orbits.

\smallskip

\noindent\emph{(i)}
If \(\mc Q^\star\) is empty, then it is trivially \(\mc G\)-stable. Let \(\mb Q^\star \in \mc Q^\star\) and take \(g \in \mc G\). Since \(\mc Q\) is \(\mc G\)-stable, we have \(g_\# \mb Q^\star \in \mc Q\). Moreover, by \eqref{eq:orbit-const},
\[
\mc D_f(\mb P \,\|\, g_\# \mb Q^\star)
=
\mc D_f(\mb P \,\|\, \mb Q^\star).
\]
Because \(\mb Q^\star\) is a minimiser over \(\mc Q\), the right-hand side equals \(\inf_{\mb Q \in \mc Q}\mc D_f(\mb P \| \mb Q)\), so \(g_\# \mb Q^\star\) also attains this infimum. Thus \(g_\# \mb Q^\star \in \mc Q^\star\), proving that \(\mc Q^\star\) is \(\mc G\)-stable.

\smallskip

\noindent\emph{(ii)}
Assume \(\mc Q^\star = \{\mb Q^\star\}\). By \emph{(i)}, for every \(g \in \mc G\) we have \(g_\# \mb Q^\star \in \mc Q^\star\). Since \(\mc Q^\star\) contains only \(\mb Q^\star\), it follows that \(g_\# \mb Q^\star = \mb Q^\star\) for all \(g \in \mc G\), i.e.\ \(\mb Q^\star\) is \(\mc G\)-invariant.
\end{proof}

We conclude the proof by finally turning to Condition~(3). As shown in Section~\ref{sec:sym_rec}, this condition guarantees that for each \(g \in \mc G\), the map \(\rho_g \colon \operatorname{Im}(S) \to \operatorname{Im}(S)\) is well defined. We now combine this observation with the preceding lemmas to obtain our statistic recovery guarantee.

\begin{theorem}[Statistic recovery, restated]\label{thm:master_theorem_appendix}
Let \(S \colon \mc P \rightharpoonup \mc Y\) be a statistic, \(\mb P \in \operatorname{dom}(S)\) a target distribution, and \(\mc Q \subseteq \operatorname{dom}(S)\) a variational family. Assume the variational problem \eqref{eq:vi_problem} has a unique minimiser, denoted by \(\mb Q^\star\). If Conditions~(1)--(3) are satisfied for some group of symmetries \(\mc G\), then
\[
S(\mb P),\,S(\mb Q^\star)\in \Gamma_S(\mc G).
\]
\end{theorem}

\begin{proof}
By Condition~(3), for every \(g \in \mc G\), the map \(\rho_g\) introduced in Section~\ref{sec:sym_rec} is well defined on \(\operatorname{Im}(S)\). Since \(\mb P\) is \(\mc G\)-invariant by Condition~(1), for every \(g \in \mc G\) we have
\[
\rho_g(S(\mb P))=S(g_\#\mb P)=S(\mb P).
\]
Hence, by definition of \(\Gamma_S(\mc G)\),
\[
S(\mb P)\in \Gamma_S(\mc G).
\]
By assumption, \(\mc Q^\star=\{\mb Q^\star\}\). By Conditions~(1) and~(2), Lemma~\ref{lem:sym_rec}\emph{(ii)} implies that \(\mb Q^\star\) is \(\mc G\)-invariant. Moreover, \(\mb Q^\star\in \mc Q\subseteq \operatorname{dom}(S)\), so \(S(\mb Q^\star)\) is well defined. Therefore, for every \(g \in \mc G\),
\[
\rho_g(S(\mb Q^\star))=S(g_\#\mb Q^\star)=S(\mb Q^\star).
\]
Again by definition of \(\Gamma_S(\mc G)\), it follows that
\[
S(\mb Q^\star)\in \Gamma_S(\mc G).
\]

Combining the two inclusions yields
\[
S(\mb P),\,S(\mb Q^\star)\in \Gamma_S(\mc G).
\]
\end{proof}

\section{Additional material for Section~\blackref{sec:prev_res}}\label{apdx:prev_res}

\subsection{Presentation differences relative to Margossian and Saul}\label{apdx:differences}

In this appendix, we discuss the differences between our presentation of the recovery guarantees in location--scale families in Section~\ref{sec:prev_res} and that of \citet{margossian_variational_2025, margossian_generalized_2025}, whose results we recover. The first difference is that we choose to work with probability measures rather than densities, in order to follow the general blueprint of Section~\ref{sec:sym_rec}. Accordingly, our definitions of location--scale families, as well as of even and elliptical symmetry, are the measure-theoretic analogues of those presented in \citet{margossian_variational_2025, margossian_generalized_2025}. This places the results in a more general setting, but at the same time forfeits some of the convenient properties that follow from the density-level formulation. For example, for covariance recovery, in addition to the necessary assumption that the target and the base distribution of the variational family have finite second moments, we explicitly assume that both have non-zero variance --- a property that automatically follows from the existence of a density with respect to Lebesgue measure. Any similar minor differences in presentation are simply artefacts of our choice to work more generally with probability distributions rather than just probability densities.

The other, and arguably more important, difference in presentation is that, whereas we state all our results under the assumption that the variational minimiser is unique, \citet{margossian_variational_2025, margossian_generalized_2025} take the extra step of imposing additional regularity conditions, such as differentiability and somewhere-strict log-concavity of the target, in order to guarantee uniqueness. Because our theory focuses on the fundamental mechanism underpinning symmetry-induced statistic recovery, we are less concerned with the specific assumptions needed to ensure uniqueness in concrete settings. Accordingly, we have chosen to state our results in the most general way possible, retaining uniqueness as a standing assumption. That said, any regularity conditions that guarantee uniqueness of the variational minimiser can be substituted for this assumption in all our corollaries.

\subsection{Correlation does not satisfy Condition~(3)}\label{apdx:corr_not_compatible}

Here, we demonstrate the problems that arise when trying to apply Theorem~\ref{thm:master_theorem} directly to the correlation statistic. First, we show that the statistic does not satisfy Condition~(3) because its domain is not \(\mc{G}\)-stable. We then show that even after restricting the domain to distributions with finite second moment and positive-definite covariance to guarantee \(\mc{G}\)-stability, the correlation statistic still fails to satisfy the condition.

Recall the setting of Section~\ref{sec:prev_res} and, in particular, the correlation statistic \(\rho \colon \mc{P}_2 \rightharpoonup \mc{C}^d_+\), defined by 
\begin{equation}\label{eq:corr_apdx}
\rho(\pi) \coloneqq D(\pi)^{-1/2}\Sigma(\pi) D(\pi)^{-1/2}, \qquad D(\pi) \coloneqq \operatorname{diag}(\Sigma(\pi)),
\end{equation}
where \( \operatorname{dom}(\rho) = \mc P_2^\circ\) is the set of distributions with finite second moment and positive marginal variances, and \(\mc{C}^d_+\) is the set of positive semidefinite matrices with unit diagonal entries. Further, recall the relevant symmetry group:
\[
\mc{G}\coloneqq \{g_R\mid R\in \mathrm{O}(d)\},
\qquad
g_R(x)\coloneqq m+A_R(x-m),
\qquad
A_R\coloneqq M^{1/2}RM^{-1/2}.
\]

Assume \(d \ge 2\) and choose \(v \in \mb{R}^d\) with \(v_i \ne 0\) for every \(i\). Define
\[
\pi \coloneqq \frac{1}{2} \delta_v + \frac{1}{2} \delta_{-v}.
\]
Then \(\pi \in \mc{P}_2\) since it is finitely supported, and moreover \(\mu(\pi) = 0\) and \(\Sigma(\pi) = vv^\top\). Hence, for all \(i \in \{1, \dots, d\}\), we have
\[
[\Sigma(\pi)]_{ii} = v_i^2 > 0.
\]
So \(\pi \in \mc{P}_2^\circ\). Set \(z \coloneqq M^{-1/2}v \neq 0\). Since \(d \ge 2\), the hyperplane
\[
(M^{1/2}e_1)^\perp
=
\{w\in \mb R^d \mid \langle w,M^{1/2}e_1\rangle =0\}
\]
has dimension \(d-1 \ge 1\), and therefore intersects the sphere of radius \(\|z\|\). Choose \(w\in (M^{1/2}e_1)^\perp\) such that \(\|w\|=\|z\|\). By transitivity of \(\mathrm{O}(d)\) on the sphere, there exists \(R\in \mathrm{O}(d)\) such that
\[
Rz=w.
\]
For this choice of \(R\),
\[
[A_Rv]_1
=
e_1^\top M^{1/2}RM^{-1/2}v
=
e_1^\top M^{1/2}w
=
\langle M^{1/2}e_1,w\rangle
=
0.
\]
Since covariance transforms under the affine map \(g_R\) by congruence with \(A_R\),
\[
\Sigma((g_R)_\#\pi)
=
A_R\Sigma(\pi)A_R^\top
=
A_Rvv^\top A_R^\top
=
(A_Rv)(A_Rv)^\top.
\]
Its first diagonal entry is therefore
\[
[\Sigma((g_R)_\#\pi)]_{11}
=
[A_Rv]_1^2
=
0.
\]
Hence \((g_R)_\#\pi\notin \mc P_2^\circ\). This proves that \(\mc P_2^\circ\) is not \(\mc G\)-stable, and thus \(\rho\) does not satisfy Condition~(3).

Since the problem already arises with the domain, one may try to suitably restrict it to ensure \(\mc{G}\)-stability, with the hope that the remaining part of the condition will go through. We now show this not to be the case. Let \(\mc{P}_2^{\Box}\) be the set of distributions with finite second moment and positive-definite covariance, and let \(\mc{C}^d_{++}\) be the set of positive-definite matrices with unit diagonal entries. Clearly \(\mc{P}_2^{\Box} \subset \mc{P}_2^\circ\) and \(\mc{C}^d_{++} \subset \mc{C}^d_+\), so we can define the correlation statistic in the same way as in \eqref{eq:corr_apdx}.

Let \(\pi\in \mc P_2^{\Box}\) and \(g_R\in \mc G\). Since \(g_R\) is affine, \((g_R)_\#\pi\in \mc P_2\). Moreover, by \eqref{eq:cov_induced_action},
\[
\Sigma((g_R)_\#\pi)=A_R\Sigma(\pi)A_R^\top.
\]
Because \(M\in \mc S_{++}^d\) and \(R\in \mathrm{O}(d)\), \(A_R\) is invertible. Therefore, since \(\Sigma(\pi)\in \mc S_{++}^d\), it follows that
\[
A_R\Sigma(\pi)A_R^\top \in \mc S_{++}^d.
\]
Thus \((g_R)_\#\pi\in \mc P_2^{\Box}\), proving \(\mc G\)-stability. However, the correlation statistic still does not satisfy Condition~(3). It suffices to give a counterexample. We work in \(\mb{R}^2\) and consider the spherical subcase \(m=0\), \(M=I_2\), for which
\[
\mc G=\{g_R\mid R\in \mathrm{O}(2)\},
\qquad
g_R(x)=Rx.
\]
Define
\[
\pi_1\coloneqq \mc N(0,I_2),
\qquad
\pi_2\coloneqq \mc N\!\left(0,
\begin{pmatrix}
1&0\\
0&3
\end{pmatrix}
\right).
\]
Clearly, both belong to \(\mc P_2^{\Box}\), and
\[
\rho(\pi_1)=I_2=\rho(\pi_2),
\]
since both covariance matrices are diagonal with positive entries. Now let
\[
R\coloneqq \frac{1}{\sqrt 2}
\begin{pmatrix}
1&-1\\
1&1
\end{pmatrix}\in \mathrm{O}(2).
\]
Then \(\Sigma((g_R)_\#\pi_1)=I_2\), so
\[
\rho((g_R)_\#\pi_1)=I_2.
\]
By contrast,
\[
\Sigma((g_R)_\#\pi_2)
=
R
\begin{pmatrix}
1&0\\
0&3
\end{pmatrix}
R^\top
=
\begin{pmatrix}
2&-1\\
-1&2
\end{pmatrix}.
\]
Therefore
\[
D((g_R)_\#\pi_2)
=
\operatorname{diag}(\Sigma((g_R)_\#\pi_2))
=
2 I_2,
\]
and hence
\[
\rho((g_R)_\#\pi_2)
=
\left(2 I_2\right)^{-1/2}
\begin{pmatrix}
2&-1\\
-1&2
\end{pmatrix}
\left(2 I_2\right)^{-1/2}
=
\begin{pmatrix}
1&-1/2\\
-1/2&1
\end{pmatrix}.
\]
Thus, there exists \(g_R \in \mc{G}\) such that
\[
\rho(\pi_1)=\rho(\pi_2)
\qquad\text{but}\qquad
\rho((g_R)_\#\pi_1)\neq \rho((g_R)_\#\pi_2).
\]
This violates the implication required by Condition~(3).

\section{Additional material for Section~\blackref{sec:sphere}}\label{apdx:sphere}

\subsection{Uniqueness of the axis of symmetry}\label{apdx:unique_axis}

In this appendix, we formally prove that if a non-uniform distribution is rotationally symmetric about two axis directions, then these determine the same line through the origin in \( \mb{R}^d \). 

\begin{proposition}[Uniqueness of the rotational axis]\label{prop:sphere-axis-unique}
Fix \(d\ge 3\). Let \(\mb \pi \in \mc{P}_{\sigma} \) be a non-uniform distribution on \(\mc S^{d-1}\) with density \(f_{\mb \pi}\), and suppose it is rotationally symmetric about both \(w_1,w_2\in \mc S^{d-1}\), i.e., there exist measurable functions \(\psi,\varphi:[-1,1]\to \mb R_{\ge 0}\) such that
\[
f_{\pi}(x)=\psi(w_1^\top x)=\varphi(w_2^\top x) \qquad\text{for \(\sigma\)-a.e. }x\in \mc S^{d-1}.
\]
Then
\[
\operatorname{span}(w_1)=\operatorname{span}(w_2).
\]
\end{proposition}

\begin{proof}
    Let \(\mathscr{T}(\pi) \coloneqq \{ R \in \mathrm{SO}(d) ~|~ (g_R)_\# \pi = \pi \}\), denote the subgroup of \(\mathrm{SO}(d)\) under which \(\pi\) is invariant. \( \mathscr{T}(\pi) \) is closed in \(\mathrm{SO}(d)\): if \(R_n \to R\) in \(\mathrm{SO}(d)\) with \((g_{R_n})_\# \pi = \pi\) for all \(n\), then for every \(h \in C(\mc{S}^{d-1})\), the functions \(x \mapsto h(R_nx)\) converge uniformly on \(\mc{S}^{d-1}\) to \(x \mapsto h(Rx)\) and hence 
    \[
    \int_{\mc{S}^{d-1}} h(x) d\pi(x) = \int_{\mc{S}^{d-1}} h(R_nx) d\pi(x) \longrightarrow \int_{\mc{S}^{d-1}} h(Rx) d\pi(x) = \int_{\mc{S}^{d-1}} h(x) d((g_R)_\# \pi)(x)
    \]
which guarantees that \((g_R)_\# \pi = \pi\) and therefore \(R \in \mathscr{T}(\pi)\). Let \(\mathscr{T}(\mb{\pi})^0 \le \mathscr{T}(\pi) \) denote the identity component of \(\mathscr{T}(\pi)\), which is a connected and closed subgroup. As \(\mathscr{T}(\pi)\) is closed in \(\mathrm{SO}(d)\), \(\mathscr{T}(\mb{\pi})^0\) is also closed in \(\mathrm{SO}(d)\). Letting \(\mc{G}_{w_i} \coloneqq \{ R \in \mathrm{SO}(d) ~|~Rw_i=w_i \}\) for \(i=1, 2\), we know that since \(\pi\) is rotationally symmetric about both \(w_1, w_2 \in \mc{S}^{d-1}\), we have \(\mc{G}_{w_1}, \mc{G}_{w_2} \subseteq \mathscr{T}(\pi)\). Moreover, \(\mathrm{SO}(d-1)\) is a maximal connected closed subgroup of \(\mathrm{SO}(d)\) \citep{dynkin_maximal_1957, uchida_classification_1979}, and \(\mc{G}_{w_1}, \mc{G}_{w_2}\) are conjugate in \(\mathrm{SO}(d)\) to the canonical embedding of \(\mathrm{SO}(d-1)\). Since conjugation is a group automorphism and a topological homeomorphism, it preserves maximal connected closed subgroups, and so \(\mc{G}_{w_1}, \mc{G}_{w_2}\) are themselves maximal connected closed subgroups of \(\mathrm{SO}(d)\). As they are connected, and both contain the identity, it follows that \(\mc{G}_{w_1}, \mc{G}_{w_2} \le \mathscr{T}(\pi)^0\).

Assume, towards a contradiction, that \(w_1 \ne \pm w_2\). Then there exists \(
R \in \mc G_{w_2} \setminus \mc G_{w_1}\). Indeed, otherwise every rotation fixing \(w_2\) would also fix \(w_1\), so \(w_1\) would belong to the common fixed-point subspace of \(\mc G_{w_2}\), namely \(\operatorname{span}(w_2)\), which would force \(w_1=\pm w_2\). Since \(\mc G_{w_2} \le \mathscr T(\pi)^0\), it follows that \(
R \in \mathscr T(\pi)^0 \setminus \mc G_{w_1}\). 
Hence \(\mc G_{w_1}\) is a proper subgroup of \(\mathscr T(\pi)^0\). But \(\mc{G}_{w_1}\) is a maximal connected closed subgroup of \(\mathrm{SO}(d)\) and so it must be that \(\mathscr{T}(\pi)^0 = \mathrm{SO}(d)\). It follows that \(\mathscr{T}(\pi) = \mathrm{SO}(d)\). Therefore \(\pi\) is invariant under every rotation of \(\mc{S}^{d-1}\). The only rotation-invariant probability measure on \(\mc{S}^{d-1}\) is the uniform measure, a contradiction. So \(\operatorname{span}(w_1) = \operatorname{span}(w_2)\).
\end{proof}

\subsection{Explicit construction of the unique minimiser example}\label{apdx:sphere_construction}

Here, we justify the claim made at the end of Section~\ref{sec:sphere} for the reverse KL divergence objective. In particular, we show that there exists a threshold \(\eta_c(d,\lambda,\kappa_0) > 0\) such that the KL has a unique minimiser for \(0 < \eta \le \eta_c(d,\lambda,\kappa_0)\), whereas for \(\eta > \eta_c(d,\lambda,\kappa_0)\) the minimiser set is non-singleton. 

Fix \(\lambda \ne 0\), \(u \in \mc S^{d-1}\), \(\kappa_0 > 0\), and \(\eta > 0\). For \(\nu \in \mc S^{d-1}\), let \(X \sim \mb Q_{\nu,\kappa_0}\), and write \(\mathscr{L}_{\lambda,\eta,\kappa_0}(\nu)\) for the reverse KL objective \(\mc D_{\mathrm{KL}}(\mb Q_{\nu,\kappa_0} \| \mb P_{\lambda,\eta})\). Using
\begin{equation*}
\log f_{\mb Q_{\nu,\kappa_0}}(x)
=
\log c_d(\kappa_0)+\kappa_0 \nu^\top x
\end{equation*}
and
\begin{equation*}
\log f_{\mb P_{\lambda,\eta}}(x)
=
-\log Z_{\lambda,\eta}
+\lambda u^\top x
-\eta (u^\top x)^2,
\end{equation*}
we obtain
\begin{align}
\mathscr{L}_{\lambda,\eta,\kappa_0}(\nu)
&=
\mb E\!\left[
\log \frac{f_{\mb Q_{\nu,\kappa_0}}(X)}{f_{\mb P_{\lambda,\eta}}(X)}
\right] \notag\\
&=
\log c_d(\kappa_0)+\log Z_{\lambda,\eta}
+
\kappa_0 \underbrace{\mb E[\nu^\top X]}_{(E_1)}
-
\lambda \underbrace{\mb E[u^\top X]}_{(E_2)}
+
\eta \underbrace{\mb E[(u^\top X)^2]}_{(E_3)}.
\label{eq:sphere_reverse_kl_objective}
\end{align}
We begin by reducing each expectation in \eqref{eq:sphere_reverse_kl_objective} to a function of the single scalar \(c \coloneqq u^\top \nu \in [-1,1]\). Set \(T \coloneqq \nu^\top X\). With \(\rho_d(t) \coloneqq (1-t^2)^{(d-3)/2}\), the density of \(T\) is \citep[Eq.~9.3.12]{mardia_directional_2009}
\begin{equation}\label{eq:sphere_density_T}
f_T(t) \propto e^{\kappa_0 t} \rho_d(t), \qquad t \in [-1,1].
\end{equation}
Now define
\[
A_{d,\kappa_0} \coloneqq \mb E[T],
\qquad
m_{d,\kappa_0} \coloneqq \mb E[T^2].
\]
Then the mean and variance of \(X\) are given by \citep[Eqs.~9.3.33--9.3.34]{mardia_directional_2009}
\begin{equation*}
\mb E[X]= A_{d,\kappa_0}\nu, \qquad 
\operatorname{Var}(X)
=
\operatorname{Var}(T)\,\nu\nu^\top
+
\frac{1-m_{d,\kappa_0}}{d-1}(I_d-\nu\nu^\top).
\end{equation*}
Therefore, the expectations \((E_1) \) and \((E_2) \) in \eqref{eq:sphere_reverse_kl_objective} are
\begin{gather}
\mb E[\nu^\top X]=A_{d,\kappa_0}, \notag\\
\mb E[u^\top X]
=
u^\top \mb E[X]
=
A_{d,\kappa_0}\,u^\top \nu
=
A_{d,\kappa_0}\,c.
\label{eq:sphere_E1_E2}
\end{gather}
Moreover, the second moment matrix is
\begin{align*}
\mb E[XX^\top]
&=
\operatorname{Var}(X)+\mb E[X]\mb E[X]^\top \notag\\
&=
\bigl(\operatorname{Var}(T)+(\mb E[T])^2\bigr)\nu\nu^\top
+
\frac{1-\mb E[T^2]}{d-1}(I_d-\nu\nu^\top) \notag\\
&=
m_{d,\kappa_0}\,\nu\nu^\top
+
\frac{1-m_{d,\kappa_0}}{d-1}(I_d-\nu\nu^\top),
\end{align*}
and hence, expectation \( (E_3) \) in \eqref{eq:sphere_reverse_kl_objective} reduces to
\begin{align}
\mb E[(u^\top X)^2]
&=
u^\top \mb E[XX^\top]u \notag\\
&=
m_{d,\kappa_0}(u^\top \nu)^2
+
\frac{1-m_{d,\kappa_0}}{d-1}\bigl(1-(u^\top \nu)^2\bigr) \notag\\
&=
\frac{1-m_{d,\kappa_0}}{d-1}
+
\underbrace{\frac{d\,m_{d,\kappa_0}-1}{d-1}}_{B_{d,\kappa_0}}\,c^2.
\label{eq:sphere_E3}
\end{align}
Substituting \eqref{eq:sphere_E1_E2} and \eqref{eq:sphere_E3} into \eqref{eq:sphere_reverse_kl_objective} yields
\begin{equation*}
\mathscr{L}_{\lambda,\eta,\kappa_0}(\nu)
=
\eta B_{d,\kappa_0}\,c^2
-
\lambda A_{d,\kappa_0}\,c
+
C_{\lambda,\eta,\kappa_0},
\end{equation*}
where
\[
C_{\lambda,\eta,\kappa_0}
\coloneqq
\log c_d(\kappa_0)+\log Z_{\lambda,\eta}
+\kappa_0 A_{d,\kappa_0}
+\eta \frac{1-m_{d,\kappa_0}}{d-1},
\]
is a constant independent of \(\nu\). We have thus reduced the optimisation problem to the scalar quadratic
\begin{equation}\label{eq:sphere_scalar_quadratic}
\widehat{\mathscr{L}}_{\lambda,\eta,\kappa_0}(c)
\coloneqq
\eta B_{d,\kappa_0}c^2-\lambda A_{d,\kappa_0}c,
\qquad c\in[-1,1].
\end{equation}

Since \(\eta > 0\), once \(B_{d,\kappa_0} > 0\) is established, the quadratic in \eqref{eq:sphere_scalar_quadratic} is strictly convex and therefore has a unique unconstrained minimiser. Because \(d \ge 3\), the denominator of \(B_{d,\kappa_0}\) is positive, so it suffices to show that
\[
d\,m_{d,\kappa_0}-1 > 0,
\]
that is, \(m_{d,\kappa_0} > 1/d\). To this end, let \(Y \sim \sigma\) be uniform on \(\mc S^{d-1}\) and define \(T_0 \coloneqq u^\top Y\). Then the density of \(T_0\) is \citep[Eq.~9.3.1]{mardia_directional_2009}
\begin{equation}\label{eq:sphere_density_T0}
f_{T_0}(t) \propto \rho_d(t), \qquad t \in [-1,1].
\end{equation}
We use \eqref{eq:sphere_density_T0} to lower bound \(m_{d,\kappa_0}\). Expanding the definition, observe that \(t^2 \rho_d(t)\) is even while \(e^{\kappa_0 t} = \cosh(\kappa_0 t) + \sinh(\kappa_0 t)\), so the odd terms integrate to \(0\) on \([-1,1]\), and we obtain
\begin{equation}
m_{d,\kappa_0}
=
\frac{\displaystyle \int_{-1}^1 t^2 e^{\kappa_0 t}\rho_d(t)\,dt}
{\displaystyle \int_{-1}^1 e^{\kappa_0 t}\rho_d(t)\,dt} =
\frac{\displaystyle \int_0^1 t^2 \cosh(\kappa_0 t)\rho_d(t)\,dt}
{\displaystyle \int_0^1 \cosh(\kappa_0 t)\rho_d(t)\,dt} 
>
\frac{\displaystyle \int_0^1 t^2 \rho_d(t)\,dt}
{\displaystyle \int_0^1 \rho_d(t)\,dt}
=
\mb E[T_0^2],
\label{eq:sphere_m_lower_bound}
\end{equation}
where the strict inequality follows from Chebyshev's integral inequality \citep[Thm.~10, Ch.~2.5]{mitrinovic_analytic_1970}, since \(t \mapsto t^2\) and \(t \mapsto \cosh(\kappa_0 t)\) are strictly increasing on \([0,1]\). To compute \(\mb E[T_0^2]\), note that since \(Y \sim \sigma\) is uniform on \(\mc S^{d-1}\), its law is rotation invariant and therefore \(\mb E[YY^\top] = \alpha I_d\) for some \(\alpha \in \mb R\). Taking traces,
\[
\tr(\mb E[YY^\top])
=
\mb E[\tr(YY^\top)]
=
\mb E[\|Y\|^2]
=
1,
\qquad
\tr(\alpha I_d) = \alpha d,
\]
and hence \(\alpha = 1/d\). The second moment of \(T_0\) is thus
\[
\mb E[T_0^2]
=
u^\top \mb E[YY^\top]u
=
u^\top \left(\frac{1}{d} I_d\right)u
=
\frac{1}{d}\|u\|^2
=
\frac{1}{d}.
\]
Combining this with \eqref{eq:sphere_m_lower_bound} gives \(m_{d,\kappa_0} > 1/d\), and therefore \(B_{d,\kappa_0} > 0\). Consequently, the scalar objective in \eqref{eq:sphere_scalar_quadratic} has the unique unconstrained minimiser
\begin{equation*}
c^\star = \frac{\lambda A_{d,\kappa_0}}{2\eta B_{d,\kappa_0}}.
\end{equation*}
Define
\begin{equation*}
\eta_c(d,\lambda,\kappa_0)
\coloneqq
\frac{|\lambda| A_{d,\kappa_0}}{2 B_{d,\kappa_0}}.
\end{equation*}
Then, \( \eta_c(d, \lambda, \kappa_0) > 0 \) since \( |\lambda|, B_{d, \kappa_0} > 0\) and moreover \(A_{d, \kappa_0} > 0\). Indeed, since \( \rho_d(t) \) is even,
\[
A_{d, \kappa_0} = \frac{\displaystyle \int_{-1}^{1}t e^{\kappa_0t} \rho_d(t) \, dt}{\displaystyle \int_{-1}^{1} e^{\kappa_0 t} \rho_d(t) \, dt} = \frac{\displaystyle \int_{0}^{1} t \left( e^{\kappa_0 t} - e^{- \kappa_0 t} \right) \rho_d(t) \, dt}{\displaystyle \int_{-1}^{1} e^{\kappa_0 t } \rho_d(t)\, dt} > 0.
\]
If \(0 < \eta \le \eta_c(d,\lambda,\kappa_0)\), then \(|c^\star| \ge 1\), and so the unique constrained minimiser of \eqref{eq:sphere_scalar_quadratic} over \([-1,1]\) is \(c = \operatorname{sgn}(\lambda)\). Since \(u^\top \nu = \operatorname{sgn}(\lambda)\) if and only if \(\nu = \operatorname{sgn}(\lambda)u\), the original objective \(\mathscr{L}_{\lambda,\eta,\kappa_0}(\nu)\) has the unique minimiser \(\nu^\star = \operatorname{sgn}(\lambda)u\). Consequently, exact statistic recovery holds with
\[
\mc A(\mb Q_{\nu^\star,\kappa_0}) = \operatorname{span}(u) = \mc A(\mb P_{\lambda,\eta}).
\]
If \(\eta > \eta_c(d,\lambda,\kappa_0)\), then \(c^\star \in (-1,1)\). Hence
\begin{equation}\label{eq:sphere_minimiser_set}
\argmin_{\nu \in \mc S^{d-1}} \mathscr{L}_{\lambda,\eta,\kappa_0}(\nu)
=
\{\nu \in \mc S^{d-1} \mid u^\top \nu = c^\star\}.
\end{equation}
Since \(d \ge 3\) and \(c^\star \in (-1,1)\), the set in \eqref{eq:sphere_minimiser_set} is a \((d-2)\)-sphere and hence is not a singleton. Moreover, no minimiser \(\nu\) satisfies \(\mc A(\mb Q_{\nu,\kappa_0}) = \mc A(\mb P_{\lambda,\eta})\), because \(\mc A(\mb Q_{\nu,\kappa_0}) = \operatorname{span}(\nu)\), while \(\operatorname{span}(\nu) = \operatorname{span}(u)\) would force \(\nu = \pm u\) and hence \(u^\top \nu \in \{\pm 1\}\), contradicting \(c^\star \in (-1,1)\).

To visualise the axis recovery, we work with \(d=3\). For fixed \(\lambda > 0\) and \(\kappa_0 > 0\), we can explicitly compute the constants \(A_{3, \kappa_0}\) and \(B_{3, \kappa_0}\) needed to evaluate \(\eta_c(3, \lambda, \kappa_0)\). First observe that since \(\rho_3(t) = 1\) for all \(t \in [-1, 1]\), the marginal density of \(T = \nu^TX\) in dimension 3 is proportional to \(e^{\kappa_0 t}\). Writing 
\[
Z(\kappa_0) \coloneqq \int_{-1}^{1} e^{\kappa_0 t} \, dt = \frac{2 \sinh \kappa_0}{\kappa_0}, 
\]
the first two moments of \(T\) can be obtained by differentiating the log-partition function:
\[
A_{3, \kappa_0} = \mb{E}[T] = \frac{d}{d \kappa_0} \log Z(\kappa_0), \qquad m_{3, \kappa_0} = \mb{E}[T^2] = \frac{d^2}{d \kappa_0^2} \log Z(\kappa_0) + \left( \frac{d}{d\kappa_0} \log Z(\kappa_0)\right)^2. 
\]
Since \(\log Z(\kappa_0) = \log 2 + \log \sinh \kappa_0 - \log \kappa_0\), this yields 
\[
A_{3, \kappa_0} = \coth \kappa_0 - \frac{1}{\kappa_0}, \qquad m_{3, \kappa_0} = \mb{E}[T^2] = 1 - \frac{2}{\kappa_0} \coth \kappa_0 + \frac{2}{\kappa_0^2}. 
\]
Substituting \(m_{3, \kappa_0}\) into the definition of \(B_{3, \kappa_0}\) gives 
\[
B_{3, \kappa_0} = 1 - \frac{3}{\kappa_0} \coth \kappa_0 + \frac{3}{\kappa_0^2}
\]
and the critical threshold therefore becomes 
\[
\eta_c(3, \lambda, \kappa_0) = \frac{\lambda \kappa_0 (\kappa_0 \coth \kappa_0 - 1)}{2(\kappa_0^2-3\kappa_0 \coth \kappa_0 + 3)}. 
\]
In the example of Figure~\ref{fig:sphere_vi}, we take \(\lambda = 1\) and \(\kappa_0 = 2.5\), so 
\begin{gather*}
A_{3, 2.5} \approx 0.6135, \qquad B_{3, 2.5} \approx 0.2637, \\ 
\eta_c(3, 1, 2.5) \approx 1.1632.
\end{gather*}
In the left panel, we use \(\eta = 1 < \eta_c(3, 1, 2.5)\), and the unique minimiser \(\nu^\star\) of \(\mathscr{L}_{\lambda, \eta, \kappa_0}(\nu)\) satisfies 
\[
\mc{A}(\mb{Q}_{\nu^\star, 2.5}) = \mc{A}(\mb{P}_{1, 1}) = \operatorname{span}(u), 
\]
while in the right panel, we use \(\eta = 2 > \eta_c(3, 1, 2.5)\), so the minimiser set is the latitude circle 
\[
\left\{
\nu\in \mc{S}^2
:
u^\top \nu
=
\frac{A_{3,2.5}}{4B_{3,2.5}}
\approx 0.5816
\right\},
\]
and \(\mc{A}(\mb{Q}_{\nu, 2.5}) \ne \mc{A}(\mb{P}_{1, 2})\) for all minimisers \(\nu\).

\end{document}